\documentclass[conference]{IEEEtran}
\usepackage{times}  
\usepackage{helvet}  
\usepackage{courier}  
\usepackage{url}  
\usepackage{graphicx}  
\usepackage{bm}
\usepackage{mathrsfs}
\usepackage{float}
\usepackage{mathtools}
\usepackage{setspace}
\usepackage{tabularx}
\usepackage{caption}
\usepackage{pifont}
\usepackage{makecell}
\usepackage{flafter}
\usepackage{algorithm}
\usepackage{algorithmic}
\usepackage{multirow,booktabs}
\usepackage{amsmath, amsthm, amssymb}

\newtheorem{proposition}{Proposition}
\newtheorem{theorem}{Theorem}
\begin{document}

\title{\Large Brain EEG Time Series Selection: A Novel Graph-Based Approach for Classification}

\author{\IEEEauthorblockN{Chenglong Dai\IEEEauthorrefmark{1},
		Jia Wu\IEEEauthorrefmark{2},
		Dechang Pi\IEEEauthorrefmark{1} and
		Lin Cui\IEEEauthorrefmark{1}\IEEEauthorrefmark{3}}
	\IEEEauthorblockA{\IEEEauthorrefmark{1}
		Computer Science and Technology, Nanjing University of Aeronautics and Astronautics, Nanjing 211106, China \\
	}
	\IEEEauthorblockA{\IEEEauthorrefmark{2}
		Department of Computing, Macquarie University, NSW 2109, Australia \\}
\IEEEauthorblockA{\IEEEauthorrefmark{3}
		Intelligent Information Processing Laboratory, Suzhou University, Suzhou 234000, China \\}
	\IEEEauthorblockA{\IEEEauthorrefmark{0} Email: \{chenglongdai, dc.pi, jsjxcuilin\}@nuaa.edu.cn;wujiawb@gmail.com  }}

\maketitle

\begin{abstract} \small\baselineskip=9pt Brain Electroencephalography (EEG) classification is widely applied to analyze cerebral diseases in recent years. Unfortunately, invalid/noisy EEGs degrade the diagnosis performance and most previously developed methods ignore the necessity of EEG selection for classification. To this end, this paper proposes a novel maximum weight clique-based EEG selection approach, named mwcEEGs, to map EEG selection to searching maximum similarity-weighted cliques from an improved Fr\'{e}chet distance-weighted undirected EEG graph simultaneously considering edge weights and vertex weights. Our mwcEEGs improves the classification performance by selecting intra-clique pairwise similar and inter-clique discriminative EEGs with similarity threshold $\delta$. Experimental results demonstrate the algorithm effectiveness  compared with the state-of-the-art time series selection algorithms on real-world EEG datasets.\end{abstract}

\begin{figure}[!t]
\centering
\includegraphics[width=3in]{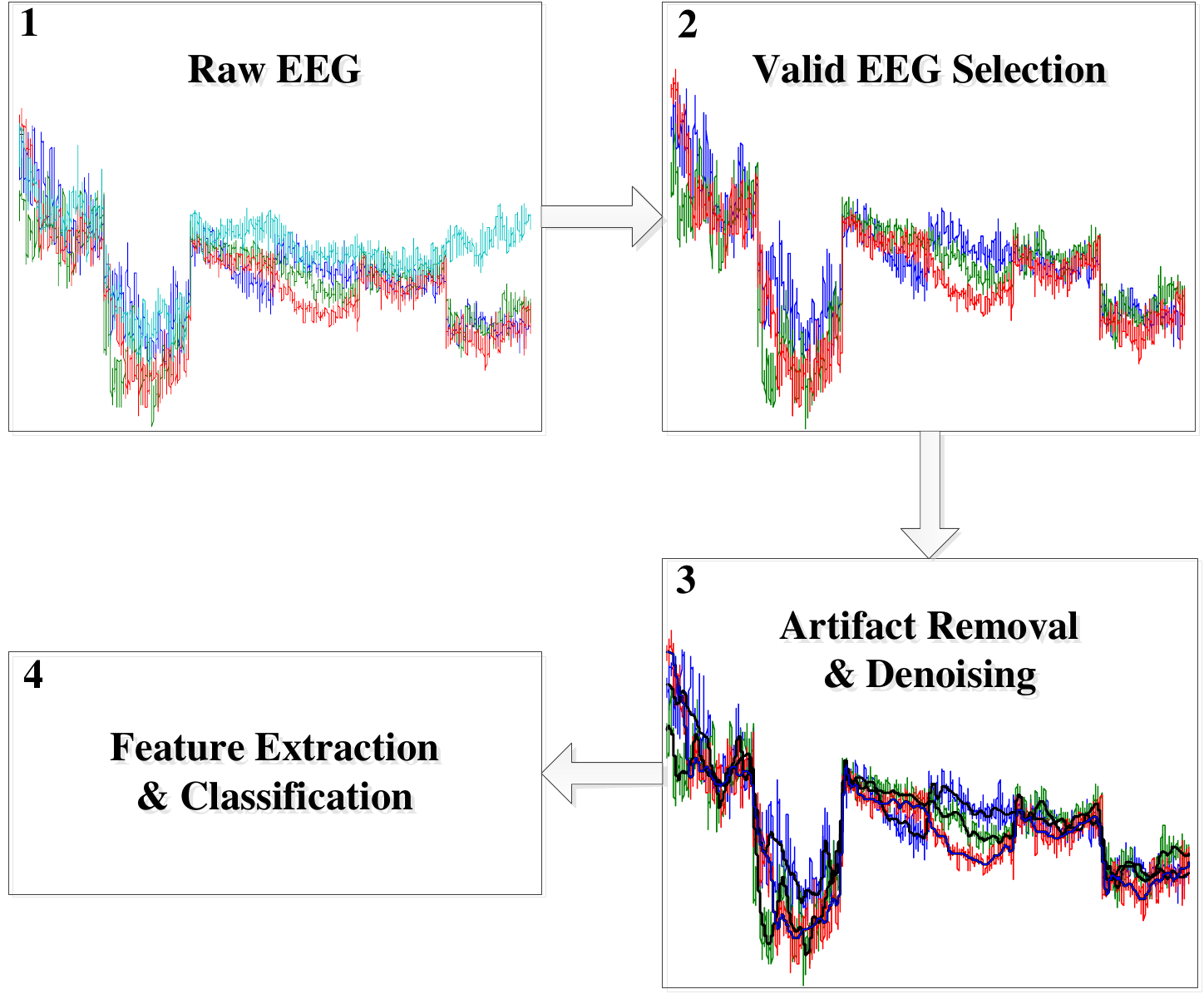}
\caption{A simplified framework of EEG analysis. 1 illustrates 4 raw EEGs including one invalid EEG (light blue); 2 shows the EEG selection that the light blue one is selected out; 3 indicates artifact removal or denoising with selected EEGs, in which the black lines demonstrate the processed results; 4 displays the final EEG classification.}
\label{fig:process}
\vspace{-0.2in}
\end{figure}

\section{Introduction} \label{introduction} In a noninvasive way, Brain Electroencephalography (EEG) classification is widely used to diagnose Alzheimer's disease \cite{Ref3}, epileptic seizure \cite{Ref23}, stroke \cite{Ref9}, and so on. But in such cerebral disease diagnosis, invalid/noisy EEGs significantly affect the diagnosis accuracy, since the invalid/noisy EEGs degrade the distinction of target features. Invalid/noisy EEG is stimulated by the non-target brain activities, whose contour shape is dissimilar with those of target ones stimulated by the specific brain activities. Invalid EEGs are mainly from (1) the environmental noises which are always ignored when analyzed and (2) the non-target bioelectrical potentials. Actually, more invalid EEGs mix in raw EEGs of those patients suffering with epilepsy, Alzheimer's disease, stroke, amyotrophic lateral sclerosis (ALS), etc., due to the uncontrolled neural actions in their brain. To improve the diagnosis accuracy, independent component analysis (ICA) \cite{Ref11}, principal component analysis (PCA) \cite{Ref10}, common spatial pattern (CSP)\cite{Ref27}, blind source separation (BSS) \cite{Ref26}, and wavelet transform (WT) \cite{Ref18} mainly consider the artifact removal and they improve the accuracy in some extent, but they ignore the impact of invalid/noisy EEGs to the follow-up analyses such as EEG artifact removal, denoising, feature extraction and classification. In another word, EEG selection is the most advance process for EEG analyses, especially for EEG classification, as Figure \ref{fig:process} illustrates. Furthermore, EEG selection aims to reduce the invalid/noisy EEGs stimulated by non-target brain activities.

EEG selection is a source control to reduce the degradation from the invalid ones. As far as we know, none of existing previous researches focused on EEG selection, and they jumped this step to artifact removal, feature extraction and classification, seeing Figure \ref{fig:process}. We study EEG selection based on maximum weight clique in the work, providing more target EEGs for EEG classification. To the best of our knowledge, this is the first try to select EEGs with maximum weight clique for its further classification.

This paper presents a novel EEG selection. It aims to map EEG selection to searching maximum weight cliques in a similarity-weighted EEG graph in such a way that EEGs in the same clique are more similar to each other than to those in different cliques. This method simultaneously considers the weights of vertices and edges in the weighted EEG graph. Meanwhile, the proposed method focuses on the correlation between pairwise EEGs in the same class and scatter among different classes. Our contributions can be summarised as follows:

$\bullet$ We present a novel method mwcEEGs for EEG selection. It maps EEG selection to searching maximum weight cliques from an similarity-weighted EEG graph, simultaneously considering the edge weights and the vertex weights.

$\bullet$ We demonstrate the superiority of mwcEEGs, with several popular and newest classifiers, over the state-of-the-art time series selection approaches through a detailed experimentation using standard classification validity criteria on the real-world EEG datasets.

The structure of this paper is as follows: In Section 2, we provide some backgrounds into similarity measure and maximum weight clique applied in this paper. In Section 3, we describe the proposed method including the EEG selection algorithm mwcEEGs and its detail description. In Section 4, we outline the datasets, criteria, and baselines to compare. The results and discussion are also presented in this section. Finally, we conclude the work in Section 5.

\section{Preliminaries} \label{preliminaries}This section introduces the backgrounds of similarity measure: Fr\'{e}chet distance and maximum weight clique problem, which are the main two parts of the proposed method.

\subsection{Fr\'{e}chet Distance}
The Fr\'{e}chet distance (FD), Hausdorff distance (HD), and dynamic time warping (DTW) are the most widely used similarity measures. FD takes into account the location and ordering of the points along the curves, which makes it a better similarity measure for EEG than HD and DTW. Since HD regards the EEG as arbitrary point sets, it ignores the order of points along the EEG. It is possible for two EEGs to have small HD but large FD \cite{Ref1}. DTW measures the distance between curves by warping the sequences in time dimension which ignores timing orders of points and degrades the synchronism of two curves. It sometimes generates unintuitive alignments and results in inferior results \cite{Ref14}, since the DTW similarity measure is not essentially positive definite. Hence, the DTW does not reflect exact similarity of two EEGs because of its time warping \cite{Ref21}. Therefore, we applied Fr\'{e}chet distance to be the similarity measure in our work.

Mathematically, for two EEGs \begin{small}$f,g\subseteq \mathbb{R}^2$\end{small} with continuous mapping \begin{small}$f,g:[0,1]\to\mathbb{R}^2$\end{small}, the Fr\'{e}chet distance \begin{small}$\delta_F(f,g)$\end{small} between \begin{small}$f$\end{small} and \begin{small}$g$\end{small} is defined as
\begin{small}
\begin{equation}
\delta_F(f,g)=\inf_{\alpha,\beta}\sup_{t\in [0,1]}\parallel f(\alpha(t))-g(\beta(t))\parallel
\end{equation}
\end{small}
where \begin{small}$||\cdot||$\end{small} is the underlying norm, and \begin{small}$\alpha,\beta:[0,1]\to[0,1]$\end{small}.

\subsection{Maximum weight Clique}
Maximum weight clique problem (MWCP) is to search a complete subgraph (any two vertices are connected by an edge) with the maximum weights of vertices or edges from a weighted graph. Mathematically, given a weighted undirected graph \begin{small}$G=(V,E,\eta,\mu)$\end{small}, where \begin{small}$V$\end{small} and \begin{small}$E$\end{small} respectively denote vertex and edge of the graph; \begin{small}$\eta:V\to \{0\}\cup\mathbb{R^+}$\end{small} and \begin{small}$\mu:E\to\{0\}\cup\mathbb{R^+}$\end{small} are respectively the weights of them. \begin{small}$\sum_{v\in V}\eta_v+\sum_{e\in E}\mu_e$\end{small} is the weight of \begin{small}$G$\end{small}. Define \begin{small}$\mathbb{N}_n=\{1,\cdots,n\},n=|V|$ and $e_{ij}=\{i,j\}\in E$\end{small}, the aim of MWCP is to search a clique \begin{small}$C$\end{small} with maximum weight from \begin{small}$G$\end{small}, see (\ref{mcp}).
\begin{equation}\label{mcp}
\small
\sum_{i=1}^n\eta_iv_i+\sum\limits_{1\leq i<j\leq n}\mu_{ij}e_{ij}\to\max_{v_i,e_{ij}\in \{0,1\}}
\end{equation}
Especially, when \begin{small}$\{i,j\}\in E_C, e_{ij}=1$\end{small}, otherwise \begin{small}$e_{ij}=0$\end{small}, the MWCP is transformed to maximum clique problem (MCP) which aims to search a complete subgraph of maximum cardinality. In the case, the MWCP is to maximize \begin{small}$\sum_{e_{ij}\in \{0,1\},i,j\in V_C}\mu_ee_{ij}$\end{small}.

\section{The proposed method}
The proposed method selects EEGs through searching maximum weight cliques in an improved Fr\'{e}chet distance weighted EEG graph simultaneously considering the weights of vertices and edges.

\subsection{Weights of Edges}
In this work, EEGs are regarded as pairwise connected vertices in an undirected weighted complete graph. The Fr\'{e}chet distance-based similarities of EEGs are the weights on the edge, also called edge weight, that determines which edge is cut when partitioning complete subgraphs. However, the conventional Fr\'{e}chet distance (CFD) ignores the temporal structure and it is sensitive to global trends \cite{Ref6}. To improve the CFD, local tendency is brought to  evaluate the trend of EEGs. Mathematically, for two EEGs \begin{small}$tr_i=(a_1,a_2,\cdots,a_m)$\end{small} and \begin{small}$tr_j=(b_1,b_2,\cdots,b_n)$\end{small}, \begin{small}$tr_i,tr_j\in\bm{Trial}$\end{small}, \begin{small}$\bm{Trial}$\end{small} is the set of EEGs, the local trend of $tr_i$, $tr_j$ is evaluated by (\ref{localtrend}).
\begin{equation}\label{localtrend}
\small
\begin{split}
LocT(tr_i,tr_j)=\frac{\sum\limits_{t=1}^{p-q}(a_{t+q}-a_t)(b_{t+q}-b_t)} {\sqrt{\sum\limits_{t=1}^{p-q}(a_{t+q}-a_t)^2\sum\limits_{t=1}^{p-q}(b_{t+q}-b_t)^2}}
\end{split}
\end{equation}
where \begin{small}$1\leq q<p=\min\{m,n\}$\end{small}. \begin{small}$q$\end{small} denotes the length of segment EEG. As \begin{small}$a_{t+q}-a_t$\end{small} indicates, a larger \begin{small}$q$\end{small} probably ignores more local tendencies than those of shorter segments whose length \begin{small}$\leq q$\end{small}. Hence, commonly \begin{small}$q=1$\end{small} or down-sampling \begin{small}$\max\{m,n\}$\end{small} into \begin{small}$p=\min\{m,n\}$\end{small}; \begin{small}$LocT(tr_i,tr_j)\in[-1,1]$\end{small} estimates the local tendency observed simultaneously on EEGs. This index indicates the synchronization of two EEGs in temporal structure.

With local and global measurements, the improved Fr\'{e}chet distance (IFD) is calculated with (\ref{similarity}) and (\ref{frechet}).
\begin{equation}\label{similarity}
\small
\begin{split}
s_{ij}=&\lambda\cdot norm(\delta_F(tr_i,tr_j))\\
&+(1-\lambda)\frac{1-LocT(tr_i,tr_j)}{2}
\end{split}
\end{equation}
where
\begin{equation}\label{frechet}
\small
\begin{split}
&\delta_F(tr_i,tr_j)=\\
&\inf\limits_{\alpha,\beta:[0,1]\to [0,1]}\sup\limits_{t\in[0,1]}S\Big(tr_i\big(\alpha(t)\big),tr_j\big(\beta(t)\big)\Big)
\end{split}
\end{equation}
and \begin{small}$norm(\delta_F(tr_i,tr_j))\in[0,1]$\end{small} is the normalized value of \begin{small}$\delta_F(tr_i,tr_j)$\end{small}; \begin{small}$\lambda\in[0,1]$\end{small} is the weight for normalized global Fr\'{e}chet distance while \begin{small}$1-\lambda$\end{small} for local tendency.

All the similarities of $n$ pairwise connected EEGs calculated by IFD with local and global measurements construct the edge weight of the undirected weighted complete graph. Let \begin{small}$\bm {Trial}$\end{small} denote a \begin{small}$n\times n$\end{small} EEG matrix \begin{small}$\bm{Trial}^{n\times n}=(tr_1,tr_2,\cdots,tr_n)$\end{small} and the diagonal normalized similarity symmetric matrix \begin{small}$\bm{S}^{n\times n}$\end{small}, the edge weights \begin{small}$\bm\mu$\end{small}, is formed as:
\begin{equation}\label{edgeweight}
\small
\begin{split}
&\bm\mu=\bm S^{n\times n}=\bm{(s_1,s_2,\cdots,s_n)}^T
\end{split}
\end{equation}
where \begin{small}$s_{ji}=s_{ij}$\end{small}; \begin{small}$s_{ij}$\end{small}, \begin{small}$s_{ji}\in \bm{s}_i$\end{small}, and \begin{small}$i$\end{small}, \begin{small}$j$\end{small} of \begin{small}$s_{ij}$\end{small} denote \begin{small}$tr_i, tr_j\in\bm{Trial},|\bm{Trial}=n|$\end{small}.

\subsection{Weights of Vertices}
Vertex weight of \begin{small}$G=(V,E,\eta,\mu)$\end{small} indicates the importance of vertex to the potential maximum similarity-weighted clique, i.e., it measures the importance of the vertex to the potential clique and also determines which vertices can be partitioned together into a same clique. For a set of EEGs \begin{small}$\bm {Trial}$\end{small} in \begin{small}$G=(V,E,\eta,\mu)$\end{small}, their importance to the potential clique can be measured by the similarity partially ordered matrix \begin{small}$\bm{\eta}$\end{small}, also called the vertex weights, such that \begin{small}$<\bm\eta^{n\times 1},\preccurlyeq_\eta>=\{<\eta_i,\eta_j>\mid \eta_i\geq \eta_j;i\neq j;i,j\in \bm {Trial}\}$\end{small}, where \begin{small}$\eta_i$\end{small} is computed by (\ref{vertexweight}).
\begin{equation}\label{vertexweight}
\small
\eta_i=\frac{1}{\mid \bm{Trial}\mid-1}\sum\limits_{j\in \bm{Trial}\setminus i}s_{ij}
\end{equation}
Specifically, \begin{small}$|\bm{Trial}|=1$\end{small}, \begin{small}$\eta=0$\end{small}. \begin{small}$\eta_i$\end{small} also represents the similarity rank of objective EEG \begin{small}$tr_i$\end{small} to the rest of ones \begin{small}$tr_j\in \bm{Trial}\setminus tr_i$\end{small}. When \begin{small}$tr_i$\end{small} is partitioned into clique \begin{small}$C$\end{small}, \begin{small}$tr_j\in\{tr_1,\cdots,tr_k\}$\end{small} with high rank similarity based on vertex weights \begin{small}$\bm\eta$\end{small} to \begin{small}$tr_i$\end{small} are correspondingly highly likely grouped into \begin{small}$C$\end{small}. In detail, the vertex weight matrix \begin{small}$\bm\eta$\end{small} is formed as
\begin{equation}\label{partialweight}
\small
\begin{split}
&\bm\eta^{n\times 1}=(\eta_1,\dots,\eta_k, \dots, \eta_n)^T
\end{split}
\end{equation}
Simultaneously considering the edge weights \begin{small}$\bm\mu^{n\times n}$\end{small} and vertex weights \begin{small}$\bm\eta^{n\times 1}$\end{small}, the pairwise high-weight EEGs with same label \begin{small}$L_i\in\bm L^{n\times1}$\end{small} grouped into the same clique \begin{small}$C_i\in\bm{\mathcal{C}}$\end{small} with respect to similarity threshold \begin{small}$\delta$\end{small} can be represented by (\ref{clique}), with which the graph partition (vertices selection) can group most similar vertices with same labels into the same clique with a minimum weight loss of edge cut. In other word, based on (\ref{clique}), the proposed method repeats selecting the vertex with highest value of \begin{small}$\bm C_i=\bm L_i\bm\eta^T_i\bm\mu_i$\end{small} such that \begin{small}$\bm\mu_i\geq \delta_i$\end{small} in class \begin{small}$L_i$\end{small} as the one adding into the clique \begin{small}$C_i$\end{small} to form a new clique with larger weight, rather than randomly selecting one, which insures vertices with high importance can be grouped together into the same clique.
\begin{equation}\label{clique}
\small
\bm{\mathcal{C}=L\eta}^T\bm\mu_{\geq \bm\delta}=\cup_{i=1}^k \bm C_i=\cup_{i=1}^k\bm L_i\bm{\eta_i}^T\bm\mu_{i\geq\delta_i}
\end{equation}

\subsection{The mwcEEGs}
In this paper, EEG selection is mapped to multi-searching cliques with maximum weight based on \begin{small}$\bm {C=L\eta^T\mu_{\geq\delta}}$\end{small} and similarity threshold $\delta$, which is named mwcEEGs. In detail, the mwcEEGs selects most similar EEGs with same label into the same clique and separates discriminative ones into different cliques with respect to \begin{small}$\delta$\end{small} and in this selecting the invalid/noisy EEGs are removed. That is to say, mwcEEGs not only selects most intra-clique similar EEGs as well as inter-clique discriminative ones, but also reduces the influence of noisy EEGs on the classifier.

Given a labeled weighted EEG graph \begin{small}$G_L=(V_L,E,\eta,\mu)$\end{small} with label matrix \begin{small}$\bm L^{n\times 1}$\end{small}, where \begin{small}$\eta:V_L\to \{0\}\cup\mathbb{R^+}$\end{small} and \begin{small}$\mu:E\to \{0\}\cup\mathbb{R^+}$\end{small}, and positive integers \begin{small}$N_1,\cdots,N_m$\end{small} such that \begin{small}$\sum_{i=1}^mN_i\leq n, n\leq|V_L|$\end{small}, the mwcEEGs aims to select a family \begin{small}$\mathcal{C}=\{C_1,\cdots,C_m\}_{\geq\delta}$\end{small} of $m_{\geq 2}$ disjoint labeled cliques with maximum weight: \begin{small}$\sum_{C_{k\geq\delta_k}\in \mathcal{C}}w_{C_k}$\end{small} based on \begin{small}$\bm{\mathcal{C}=L\eta^T\mu_{\geq\delta}}$\end{small} with respect to the similarity threshold $\delta$. Given any edge \begin{small}$(i,j)\in E(C_k)$\end{small}, \begin{small}$k\in \mathbb{N}_m$\end{small}, and the weight of edge in \begin{small}$C_k$\end{small} such that \begin{small}$\mu_{ij}\in\{s_{ij}|s_{ij}\geq\delta_{C_k}\}$\end{small}, the weight function simultaneously considering the weights of edges and vertices is modified as (\ref{modifiedweight}) defines.
\begin{equation}\label{modifiedweight}
\small
w_{ij}=\begin{cases}
\frac{\eta_i+\eta_j}{N_k-1}+\mu_{ij}, &\text{if } \mu_{ij}\geq\delta_{C_k},i\neq j;\\
\qquad 0, & \text{otherwise.}
\end{cases}
\end{equation}
Where \begin{small}$N_k$\end{small} denotes the number of vertices of clique \begin{small}$C_k$\end{small}. \begin{small}$N_k$\end{small} increases along with a new vertex \begin{small}$v_t$\end{small} such that \begin{small}$\forall v_i\in C_k, \mu_{ti}\geq\delta_{C_k}$\end{small} joins into \begin{small}$C_k$\end{small} and then the modified weighted \begin{small}$w_{ij}$\end{small} is correspondingly updated. Moreover, the similarity thresholds \begin{small}$\delta$\end{small} are crucial to EEG selection. Simply, similar vertices whose edge weights larger than corresponding \begin{small}$\delta_k$\end{small} are likely grouped into the same clique \begin{small}$C_k$\end{small}. Namely, \begin{small}$\delta$\end{small} influences the EEG selection results. Furthermore, for searching \begin{small}$m$\end{small} cliques, it needs \begin{small}$m-1$\end{small} thresholds, since \begin{small}$m-1$\end{small} thresholds achieve \begin{small}$m-1$\end{small} cliques while remaining vertices are naturally regarded as a clique, then \begin{small}$m$\end{small} cliques are finally achieved.

\begin{proposition}\label{proposition1} The mwcEEGs with the modified weight function (\ref{modifiedweight}) can be written equivalently as\\
\begin{equation}
\small
\mathcal{F(C)}=\sum_{k=1}^mF(C_k)=\sum_{k=1}^m\sum_{\substack{e\in E(C_k),\delta_0=1\\ \delta_k\leq\mu_e<\delta_{k-1}}}w_e\to\max_{\mathcal{C}\in\mathscr{C}}\notag\\
\end{equation}
or
\begin{equation}
\mathcal{F(C)}=\frac12\sum_{k=1}^m\sum_{i\in C_k}\sum_{\substack{j\in C_k,\delta_0=1\\ \delta_k\leq\mu_{ij}<\delta_{k-1}}}w_{ij}\to \max_{\mathcal{C}\in\mathscr{C}}\notag
\end{equation}
where \begin{small}$\mathscr{C}=\Big\{\mathcal{C}=\{C_1,\cdots,C_m\}: C_i\cap C_j=\varnothing;|C_i|=N_i,|C_j|=N_j;\{i,j\}\in\mathbb{N}_m\Big\}$\end{small}, $k$ denotes the EEG class index.\end{proposition}
\begin{proof} Set \begin{small}$A_k=A(C_k)=\sum\nolimits_{i\in C_k}\eta_i$\end{small}, then
\begin{small}
\begin{align}
\begin{split}
&\mathcal{F(C)}=\sum_{k=1}^m\sum_{\substack{e\in E(C_k),\delta_0=1\\ \delta_k\leq\mu_e<\delta_{k-1}}}w_e=\frac12\sum_{k=1}^m\sum_{i\in C_k}\sum_{\substack{j\in C_k\backslash i,\delta_0=1\\ \delta_k\leq\mu_{ij}<\delta_{k-1}}}w_{ij} \notag\\
={} &\frac12\sum_{k=1}^m\sum_{i\in C_k}\sum_{\substack{j\in C_k\backslash i,\delta_0=1\\ \delta_k\leq\mu_{ij}<\delta_{k-1}}}\Big(\frac{\eta_i+\eta_j}{N_k-1}+\mu_{ij}\Big)\notag\\
={}&\frac12\sum_{k=1}^m\sum_{i\in C_k}\Big(\sum_{j\in C_k\backslash i}\frac{\eta_i}{N_k-1}+\\
&~~~~~~~~~~~~~~~~~~~\sum_{j\in C_k\backslash i}\frac{\eta_j}{N_k-1}+\sum_{\substack{j\in C_k\backslash i,\delta_0=1\\ \delta_k\leq\mu_{ij}<\delta_{k-1}}}\mu_{ij}\Big)\notag\\
={}&\frac12\sum_{k=1}^m\sum_{i\in C_k}\Big(\eta_i+\frac{A_k-\eta_i}{N_k-1}+\sum_{\substack{j\in C_k\backslash i,\delta_0=1\\ \delta_k\leq\mu_{ij}<\delta_{k-1}}}\mu_{ij}\Big)\notag\\
=&\sum_{k=1}^m\Big(\frac{(N_k-2)\eta_i+A_k}{2(N_k-1)}+\frac12\sum_{i\in C_k}\sum_{\substack{j\in C_k\backslash i,\delta_0=1\\ \delta_k\leq\mu_{ij}<\delta_{k-1}}}\mu_{ij}\Big)\notag\\
=&\sum_{k=1}^m\Big(\sum_{i\in C_k}\eta_i+\sum_{\substack{e\in E(C_k),\delta_0=1\\ \delta_k\leq\mu_e<\delta_{k-1}}}\mu_e\Big)\notag
\end{split}
\end{align}
\end{small}
\end{proof}

Proposition \ref{proposition1} indicates that searching the maximum similarity-weighted cliques based on the modified weight function is to maximize the total weight of vertices and edges satisfying the similarity thresholds. With the modified weight function simultaneously considering edge weights and vertex weights, the pseudo-code of mwcEEGs for EEG selection is shown in Algorithm \ref{mwcEEGs1}. The mwcEEGs firstly sets \begin{small}$m-1$\end{small} similarity thresholds, vertex weight matrix \begin{small}$\bm\eta$\end{small} with (\ref{similarity},\ref{frechet},\ref{edgeweight}) and edge weight matrix $\bm\mu$ with (\ref{vertexweight},\ref{partialweight}) for initializing the labeled EEG graph \begin{small}$G_L=(V_L,E,\bm\eta,\bm\mu)$\end{small}, seeing line 1. Then it calculates the value of \begin{small}$\bm C_k$\end{small} and selects the vertex with the maximum value into the $k^{th}$ clique without randomly selecting one, see lines 4-5. Subsequently, mwcEEGs compares the total weight of the new clique \begin{small}${C_k\cup\{v_t\}}$\end{small} with the old one \begin{small}$C_k$\end{small} to determine the new vertex \begin{small}$v_t$\end{small} joining to the clique or not, as lines 6-13 indicate. Additionally with lines 9-13, the matrix \begin{small}$\bm C_k$\end{small}, \begin{small}$\bm\eta_k$\end{small} and \begin{small}$\bm\mu_k$\end{small} are updated based on the vertex adding. When the \begin{small}$k^{th}$\end{small} clique \begin{small}$C_k$\end{small} whose vertex set is \begin{small}$V_{C_k}$\end{small} is searched out, the $\bm\eta_k$ and $\bm \mu_k$ are modified correspondingly in lines 9-13, to calculate the weight matrix of the remaining vertices \begin{small}$V-V_{C_k}$\end{small} and their edges with same label $l_k$, and then the vertex with $l_k$ holding the largest value in \begin{small}$\bm C_k$\end{small} achieved by \begin{small}$\bm\eta_k\bm\mu_k$\end{small} is most likely chosen as the next vertex into the clique, to form a new clique \begin{small}$C_k$\end{small} with larger total weight.

As Algorithm \ref{mwcEEGs1} and Proposition \ref{proposition1} demonstrate, the algorithm selects labeled EEGs \begin{small}$v_i$\end{small} such that \begin{small}$\forall v_j\in C_k,\mu_{ij}\geq\delta_k$\end{small} and then trains the classifier model with such selected labeled EEGs. In other words, this process with respect to \begin{small}$\delta_k$\end{small} not only chooses the most distinguished labeled EEGs with high similarity to train the classifier model, but also reduces the influence of invalid/noisy EEGs.

In the algorithm, a vertex \begin{small}$v_t$\end{small} labeled \begin{small}$l_k$\end{small} joining the clique \begin{small}$C_k$\end{small} if it simultaneously satisfies 2 conditions: (1) \begin{small}$\forall v_j\in C_k$\end{small}, \begin{small}$\mu_{tj}\geq\delta_k$\end{small}; (2) \begin{small}$\sum\limits_{v_i,v_j\in C_k\cup\{v_t\}}w_{ij}\geq\sum\limits_{v_i,v_j\in C_k}w_{ij}$\end{small}. Actually, once $\delta_k$ is set, adding $v_t$ into $C_k$ just needs satisfying (1).
\begin{theorem}\label{theorem} A vertex \begin{small}$v_t$\end{small} labeled \begin{small}$l_k$\end{small} joins the clique \begin{small}$C_k$\end{small} to construct a new larger-weight clique \begin{small}$C_k^\ast$\end{small} if and only if \emph{\begin{small}$\forall v_i\in C_k,\mu_{it}\geq\delta_k$\end{small}.}\end{theorem}
\begin{proof} Since \begin{small}$v_t\in C_k^\ast$\end{small}, obviously \begin{small}$\forall v_i\in C_k,\mu_{it}\geq\delta_k$\end{small}, and \begin{small}$l_k=l_{C_k}$\end{small}. For a vertex \begin{small}$v_t$\end{small} labeled \begin{small}$l_k$\end{small} such that \begin{small}$\forall v_j\in C_k,\mu_{tj}\geq\delta_k$\end{small}, according to Proposition \ref{proposition1}: \begin{small}$\sum_{e\in E(C_k)}w_e=\sum_{v_i\in C_k}\eta_i+\sum_{e\in E(C_k)}\mu_e$\end{small}, then
\begin{small}
\begin{align}
\begin{split}
&\sum_{e\in E(C_k\cup\{v_t\})}w_e-\sum_{e\in E(C_k)}w_e \notag\\
={}&\sum_{v_i\in C_k\cup\{v_t\}}\eta_i+\sum_{e\in E(C_k\cup\{v_t\})}\mu_e-\sum_{v_i\in C_k}\eta_i-\sum_{e\in E(C_k)}\mu_e \notag\\
={} & \sum_{v_i\in C_k}\eta_i+\eta_t+\sum_{e\in E(C_k)}\mu_e+\sum_{v_i\in C_k}\mu_{it}\\
&-\sum_{v_i\in C_k}\eta_i-\sum_{e\in E(C_k)}\mu_e \notag= \eta_t+\sum_{v_i\in C_k}\mu_{it} \quad>0
\end{split}
\end{align}
\end{small}

Namely, \begin{small}$v_t$\end{small} joining \begin{small}$C_k$\end{small} increases the weight, therefore \begin{small}$v_t$\end{small} and \begin{small}$\forall v_i\in C_k$\end{small} construct a new clique \begin{small}$C_k^\ast=C_k\cup \{v_t\}$\end{small} with a larger weight than that of \begin{small}$C_k$\end{small}.\end{proof}
Based on Theorem \ref{theorem}, once \begin{small}$\delta_{k=1,\cdots,m}$\end{small} is set, searching \begin{small}$m$\end{small} labeled cliques such that \begin{small}$\max_{\mathcal{C}\in\mathscr{C}}$\end{small} can be transformed to \begin{small}$m$\end{small}-search labeled maximum similarity-weighted cliques, namely \begin{small}$m$\end{small}-repeating mwcEEGs with \begin{small}$\delta_k$\end{small}, \begin{small}$k=1,\cdots,m$\end{small}.
\begin{proposition}\label{proposition2} The mwcEEGs for \begin{small}$m$\end{small} labeled cliques selection can be equivalently transformed to \begin{small}$m$\end{small}-time repeating mwcEEGs with \begin{small}$0\leq\delta_k<\delta_{k-1}, \delta_0=1$\end{small}. \end{proposition}
\begin{proof} Recall Proposition \ref{proposition1}. The mwcEEGs, can be written as \begin{small}$\mathcal{F(C)}=\sum\limits_{k=1}^m\sum\limits_{\substack{e\in E(C_k),\delta_0=1\\\delta_k\leq\mu_e<\delta_{k-1}}}w_e\to \max\limits_{\mathcal{C}\in\mathscr{C}}$\end{small}, then
\begin{small}
\begin{align}
\begin{split}
&\mathcal{F(C)}=\sum_{k=1}^m\sum_{\substack{e\in E(C_k),\delta_0=1\\\delta_k\leq\mu_e<\delta_{k-1}}}w_e\to\max\limits_{\mathcal{C}\in\mathscr{C}}\notag \\
={} &\sum_{k=1}^m\Big(\sum_{i\in C_k}\eta_i+\sum_{\substack{e\in E(C_k),\delta_0=1\\\delta_k\leq\mu_e<\delta_{k-1}}}\mu_e\Big)\to \max\limits_{\mathcal{C}\in\mathscr{C}}\notag\\
={}&\Big((\sum_{i\in C_1}\eta_i+\sum_{e\in E(C_1),\delta_1\leq\mu_e<1}\mu_e)+\cdots\notag\\
&\quad+(\sum_{i\in\mathcal{C}\backslash C_1\backslash\cdots\backslash C_{m-1}}\eta_i+\sum_{\substack{e\in E(\mathcal{C}\backslash C_1\backslash \cdots\backslash C_{m-1})\\\delta_m\leq \mu_e<\delta_{m-1}}}\mu_e)\Big)\to\max\limits_{\mathcal{C}\in\mathscr{C}}\notag\\
&\because e\in E(C_k), 0\leq\delta_k\leq e <\delta_{k-1},\delta_0=1; k=1,\cdots,m\notag\\
={}&\Big((\sum_{i\in C_1}\eta_i+\sum_{e\in E(C_1),\delta_1\leq\mu_e<1}\mu_e)\to\max\limits_{C_1\in\mathcal{C}}\Big)+\cdots\notag\\
&+\Big((\sum_{i\in\mathcal{C}\backslash C_1\backslash\cdots\backslash C_{m-1}}\eta_i+\sum_{\substack{e\in E(\mathcal{C}\backslash C_1\backslash \cdots\backslash C_{m-1})\\\delta_m\leq \mu_e<\delta_{m-1}}}\mu_e)\to\max\limits_{C_m\in\mathcal{C}}\Big)\notag\\
={}&\sum_{k=1}^m\Big\{(\sum_{i\in C_k}\eta_i+\sum_{\substack{e\in E(C_k),\delta_0=1\\\delta_k\leq\mu_e<\delta_{k-1}}}\mu_e)\to \max_{C_m\in\mathcal{C}}\Big\}
\end{split}
\end{align}
\end{small}
\end{proof}

\begin{algorithm}[!t]
\begin{small}
\renewcommand{\algorithmicrequire}{ \textbf{Input:}}
\renewcommand{\algorithmicensure}{ \textbf{Output:}}
\caption{\quad Graph-based EEG Selection.}
\begin{algorithmic}[1]\label{mwcEEGs1}
\REQUIRE~~\\
 $\bm\delta$: Set of $m_{\geq1}$ similarity thresholds $\bm\delta=\{\delta_k|0\leq\delta_k<\delta_{k-1};\delta_0=1,k=1,\cdots,m\}$;\\
 $\bm L^{n\times 1}$: label matrix of $n$ EEGs;\\
 $\bm\eta^{n\times 1}$, vertex weight matrix of $n$ EEGs;\\
 $\bm\mu^{n\times n}$, edge weight matrix of $n$ EEGs;\\
\ENSURE~~\\
$\mathcal{C}$: set of $m$ labeled cliques such that $\max\limits_{\mathcal{C}\in\mathscr{C}}\mathcal{F(C)}$;
\STATE Initialize labeled EEG graph $G_L=(V_L,E,\bm\eta,\bm\mu)$ and $\mathcal{C}=\varnothing$;
\REPEAT
\STATE $\delta_k\in\bm\delta$, $\bm L_k=\bm L$;
\STATE $\bm C_k=\bm L_k\bm\eta^T_k\bm\mu_k$;
\STATE $\{v_t$ labeled $l_k$ with maximum value in $\bm C_k\}\in V_L$;
\IF {$\forall \{v_n\}\in V_{C_k}, \delta_k\leq\mu_{v_tv_n}<\delta_{k-1}$}
\STATE $W_{C_k\cup\{v_t\}}=\sum\nolimits_{i,j\in{V_{C_k}\cup\{v_t\}}}w_{ij}$;
\ENDIF
\IF {$W_{C_k\cup\{v_t\}}\geq W_{C_k}$}
\STATE $V_{C_k}=V_{C_k}\cup\{v_t\}$;
\STATE $V_L=V_L\setminus \{v_t\}$;
\STATE $W_{C_k}=W_{C_k\cup\{v_t\}}$;
\ENDIF
\STATE $\mathcal{C}=\mathcal{C}\cup C_k$;
\STATE $\bm L=\bm L\setminus\bm L_k$, $\bm\delta=\bm\delta\setminus\{\delta_k\}$;
\STATE $k=k+1$;
\UNTIL{$\bm\delta=\varnothing$ or $\bm L=\varnothing$};
\end{algorithmic}
\end{small}
\end{algorithm}

According to Proposition \ref{proposition2} and Theorem \ref{theorem}, the mwcEEGs is transformed by \begin{small}$m$\end{small}-searching labeled maximum cliques, as Algorithm \ref{mwcEEGs2} shows, in which any algorithm for maximum clique problem (MCP) \cite{Ref28} can be applied to search the cliques with maximum weight in the given graph. Importantly, in every iteration to search the maximum clique, the vertex weights \begin{small}$(\eta_v\in\bm\eta_{V_k})$\end{small} of \begin{small}$v\in V_k$\end{small} will be ranked in a descending order and the weight matrix \begin{small}$(\bm C_k=\bm L_k\bm\eta_k^T\bm\mu_k)$\end{small} is calculated in line 8, so that the algorithm can choose vertex with highest weight into the potential clique \begin{small}$C_k$\end{small}. This procedure contributes high-quality selection and fast searching maximum clique. A higher \begin{small}$\eta_i \in \bm \eta$\end{small} indicates the vertex \begin{small}$v_i$\end{small} has higher similarity with all the other vertices. That is, vertices with higher \begin{small}$\eta_i$\end{small} are likely grouped into the same clique. Meanwhile, selecting the vertex with largest value in \begin{small}$\bm C_k$\end{small} as the new vertex adding to the potential clique also reduces the time consumption compared with the conventional methods that randomly select vertices.

\section{Experiments} \label{experiments}

\subsection{Datasets}

The EEG data we experiment with are slow cortical potentials (SCPs \footnote{The data set is publicly available as online archives at http://www.bbci.de/competition/ii/.}) provided by Institute of Medical Psychology and Behavioral Neurobiology from University of T\"ubingen. In detail, Dataset Ia with 135 EEG trials labeled '0' (Traindata\_0, Ia) and 133 labeled '1' (Traindata\_1, Ia) are taken from a healthy subject (HS). Dateset Ib with 100 EEG trials labeled '0' (Traindata\_0, Ib) and 100 labeled '1' (Traindata\_1, Ib) are taken from an amyotrophic lateral sclerosis subject (ALS). 3 cases of experiments are set up in Table \ref{tab:Datasets}. In this paper, we apply Hold-out strategy \cite{Ref15} to evaluate methods. The data sets are divided into two parts: training data and testing data with the proportion of 2:1. Furthermore, the Hold-out strategy is applied 3 times to produce 3 groups of training and testing data for the methods.

\setlength\abovecaptionskip{0.2cm}
\setlength\belowcaptionskip{0cm}
\begin{table}[!h]
\newcommand{\tabincell}[2]{\begin{tabular}{@{}#1@{}}#2\end{tabular}}
\centering
\caption{Datasets}
\label{tab:Datasets}
\scalebox{0.69}{
\begin{tabular}{|c|c|c|c|c|}
\hline
EEG Cases & Datasets &  Training:Testing & \# of Classes\\
\hline
\hline
EEG Case 1 & Traindata\_0, Ia + Traindata\_1, Ia & 180 : 88 (268) & 2\\
\hline
EEG Case 2 & Traindata\_0, Ib + Traindata\_1, Ib & 134 : 66 (200) & 2\\
\hline
EEG Case 3 &  \tabincell{c}{Traindata\_0 + Traindata\_1, Ia \\Traindata\_0 + Traindata\_1, Ib} & \tabincell{c} {313 : 156} (468) & 4\\
\hline
\end{tabular}}
\vspace{-0.2in}
\end{table}

\begin{algorithm}[!t]
\begin{small}
\renewcommand{\algorithmicrequire}{ \textbf{Input:}}
\renewcommand{\algorithmicensure}{ \textbf{Output:}}
\caption{\quad mwcEEGs via Maximum Weight Cliques}
\begin{algorithmic}[1]\label{mwcEEGs2}
\REQUIRE~~\\
 $\bm\delta$: Set of $m_{\geq1}$ similarity thresholds $\bm\delta=\{\delta_k|0\leq\delta_k<\delta_{k-1};\delta_0=1,k=1,\cdots,m\}$;\\
  $\bm L^{n\times 1}$: label matrix of $n$ EEG;\\
 $\bm\eta^{n\times 1}$, vertex weight matrix of $n$ EEG;\\
 $\bm\mu^{n\times n}$, edge weight matrix of $n$ EEG;\\
\ENSURE~~\\
 $\mathcal{C}$: set of $m$ labeled cliques such that $\max\limits_{\mathcal{C}\in\mathscr{C}}\mathcal{F(C)}$;
\STATE Initialize labeled EEG graph $_LG=(V_L,E,\eta,\mu)$ and $\mathcal{C}=\varnothing$;
\REPEAT
\STATE $G_k=G$: $V_k=V_L$, $E_k=E$;
\STATE $\delta_k\in\bm\delta$, $\bm L_k=\bm L$;
\IF {$\mu_{ij}<\delta_k$}
\STATE $E_k=E_k\setminus\{e_{ij}\}$;
\ENDIF
\STATE Search $C_k=(V_C,E_C)$ such that $\max |V_C|$ from $G_k=(V_k,E_k,\eta,\mu)$ with any MCP algorithm, improved with descending sorting $(\eta_v\in\bm\eta_{V_k})$ of $v\in V_k$ and $C_k=\bm L_k\bm\eta^T_k\bm\mu_k$;
\STATE $\mathcal{C}=\mathcal{C}\cup C_k$;
\STATE $G_M=(V_M,E_M,\eta,\mu)$:\\
 \qquad$V_M=V_L-V_C=\{i\mid i\in V_L;i\notin V_C\}$,\\
 \qquad $E_M=\{e_{ij}\mid i,j\in V_L;i,j\notin V_C;i\neq j\}$;
\STATE $G=G_M$;
\STATE $\bm L=\bm L\setminus\bm L_k$, $\bm\delta=\bm\delta\setminus\{\delta_k\}$;
\STATE $k=k+1$;
\UNTIL{$\delta=\varnothing$ or $\bm L=\varnothing$};
\end{algorithmic}
\end{small}
\end{algorithm}

\subsection{Evaluation Methodology}
The following criteria, which have been well used in data mining area \cite{Ref5,Ref8,Ref20}, are selected to evaluate the proposed method.

$\bullet$ \textbf{\emph{Rand index}} ({\em RI}) \cite{Ref22} estimates the quality of classification with respect to the right classes of the data. It measures the percentage of right decisions made by the method. In detail, \begin{small}$RI=\frac{TP+TN}{TP+TN+FP+FN}$\end{small}, where \emph{TP}, \emph{FP}, \emph{TN}, and \emph{FN} respectively denotes the number of true positives, false positives, true negatives, and false negatives.

$\bullet$ \textbf{\emph{F-score}} \cite{Ref25} weighs FP and FN in RI unequally through weighting a parameter \begin{small}$\beta\geq0$\end{small} on \emph{recall}, commonly \begin{small}$\beta=1$\end{small}. Mathematically, \emph{F-score}\begin{small}$=\frac{(1+\beta^2)pr}{\beta^2p+r}$\end{small}, where \emph{precision}: \begin{small}$p=\frac{TP}{TP+FP}$\end{small} and \emph{recall}: \begin{small}$r=\frac{TP}{TP+FN}$\end{small}.

$\bullet$ \textbf{\emph{Fleiss' kappa ($\kappa$)}} \cite{Ref7} is a statistical measure for assessing the coherence of decision ratings among classes. Mathematically, \begin{small}$\kappa=\frac{\overline{P}-\overline{P}_e}{1-\overline{P}_e}$\end{small}, where \begin{small}$\overline{P}-\overline{P_e}$\end{small} denotes the degree of agreement actually achieved over chance and \begin{small}$1-\overline{P}_e$\end{small} denotes the degree of agreement attainable above chance. Meanwhile, \begin{small}$\overline{P}=\frac{1}{Nn(n-1)}(\sum_{i=1}^N\sum_{j=1}^kn_{ij}^2-Nn)$\end{small}, \begin{small}$\overline{P}_e=\sum_{j=1}^k(\frac{1}{Nn}\sum_{i=1}^Nn_{ij})^2$\end{small}, and \begin{small}$N$\end{small} denotes the number of subjects, \begin{small}$n$\end{small} the number of ratings per subject, \begin{small}$k$\end{small} the number of classes into which assignment are made.

\subsection{Baselines}
We compared mwcEEGs with the state-of-the-art EEG time series selection methods, as follows.

$\bullet$ \textbf{lwEEGs}: Local weighted EEG time series selection computes time series centroid in each class and selects \begin{small}$k$\end{small} nearest \cite{Ref12,Ref13} time series from the same labeles to the corresponding centroid as the training time series.

$\bullet$ \textbf{gwEEGs}: Global weighted EEG time series selection computes the centroid of all labeled time series and selects \begin{small}$k$\end{small} closest ones from all classes to the centroid time series as the training ones for classifiers.

$\bullet$ \textbf{lrtEEGs} \cite{Ref29}: Local recursion testing EEG time series selection recursively selects \begin{small}$m$\end{small} nearest time series to every testing one with same label and chooses the \begin{small}$k\leq m$\end{small} most nearest ones from \begin{small}$m$\end{small} selected time series in each class as the training data for classifiers, which focuses on the local correlation between time series to the testing ones.

$\bullet$ \textbf{grtEEGs} \cite{Ref29}: Global recursion testing EEG time series selection recursively selects \begin{small}$m$\end{small} nearest time series to each testing time series without considering the class labels and then chooses the \begin{small}$k\leq m$\end{small} most similar ones as the training time series for classifiers. It considers the global similarity between all the time series to the entire testing time series.

Meanwhile, in order to evaluate all the methods to select EEGs, we apply most popular and newest classifiers to classify EEGs with EEG selection methods. The applied classifiers are introduced below, which mainly include most widely applied SVM, shapelet-based, ensemble-based, and structure-based classifiers.

$\bullet$ \textbf{SVM}: We apply LIBSVM \cite{Ref4} in this section as one of the baselines to classify EEG data. With LIBSVM, the kernel width and \begin{small}$C$\end{small} of SVM are respectively tuned as \begin{small}$\gamma\in\{10^{-3},10^{-2},\dots,10^1\}$\end{small} and \begin{small}$C\in\{10^{-2},10^{-1},\dots,10^1,10^2\}$\end{small}.

$\bullet$ \textbf{st-TSC} \cite{Ref17}: Shapelet transform-based method time series classifier extracts \begin{small}$k$\end{small} discriminative subsequences that best distinguish time series in different classes and uses an optimization formulation to search for fixed length time series subsequences that best predict the target variable by calculating the distances from a series to each shapelet.

$\bullet$ \textbf{RPCD} \cite{Ref24}: Recurrence patterns compression distance time series classifier uses recurrence plots as representation domain for time series classification via applying Campana-Keogh distance to estimate similarity.

$\bullet$ \textbf{COTE} \cite{Ref2}: An ensemble-based classifier classifies time series by applying a heterogeneous ensemble onto transformed representations. A flat collective of transform based ensembles (COTE) fuses various classifiers into a single one, which includes whole time series classifiers, shapelet classifiers, and spectral classifiers.

$\bullet$ \textbf{SAX-SEQL} \cite{Ref19}: An efficient linear classifier learns long discrete discriminative subsequences from time series by exploiting all-subsequences space based on symbolic aggregate approximation (SAX) \cite{Ref16} which smooths and compresses time series into discrete representations.

\begin{figure}[!t]
\centering
\includegraphics[height=4.4in]{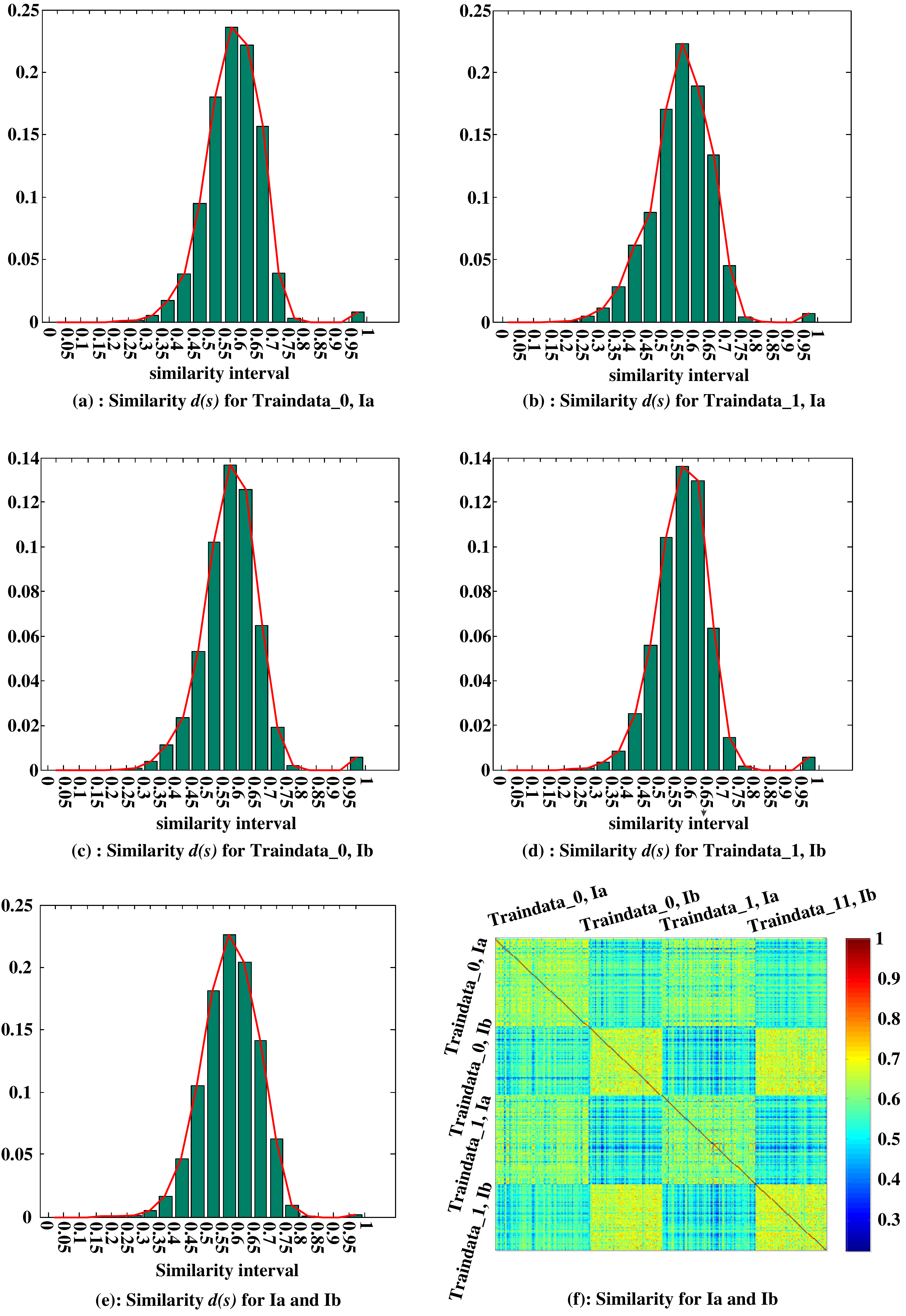}
\caption{Similarity distribution $d(s)$ of 3 EEG datasets.}
\label{fig:density}
\vspace{-0.4cm}
\end{figure}

\begin{figure*}[!t]
\centering
\includegraphics[height=3.6in]{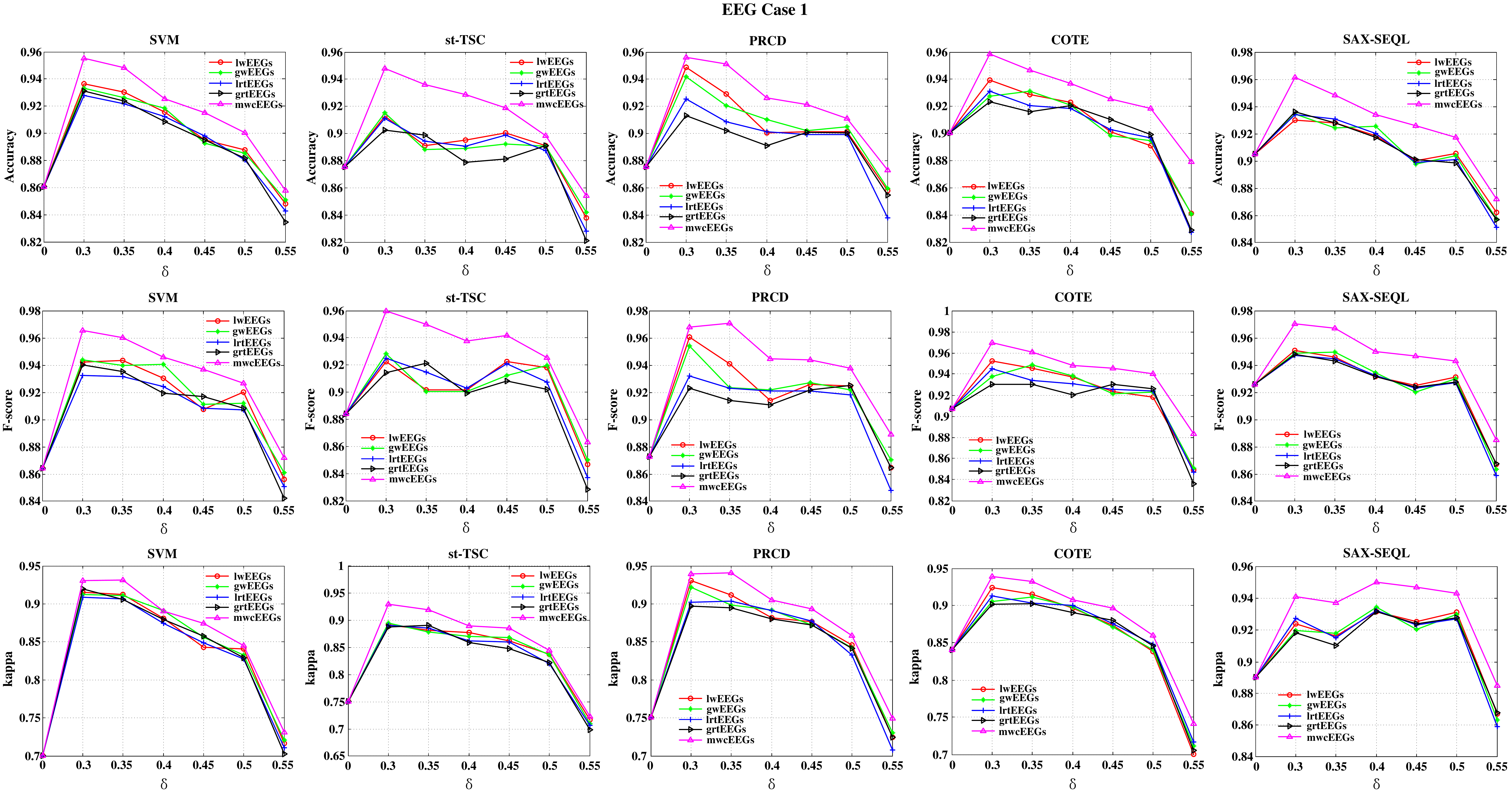}
\caption{Classification results of 5 classifiers with selected EEGs by different EEG selection methods on EEG Case 1.}
\label{fig:case1}\vspace{-0.2cm}
\end{figure*}

\begin{figure*}[!t]
\centering
\includegraphics[height=3.6in]{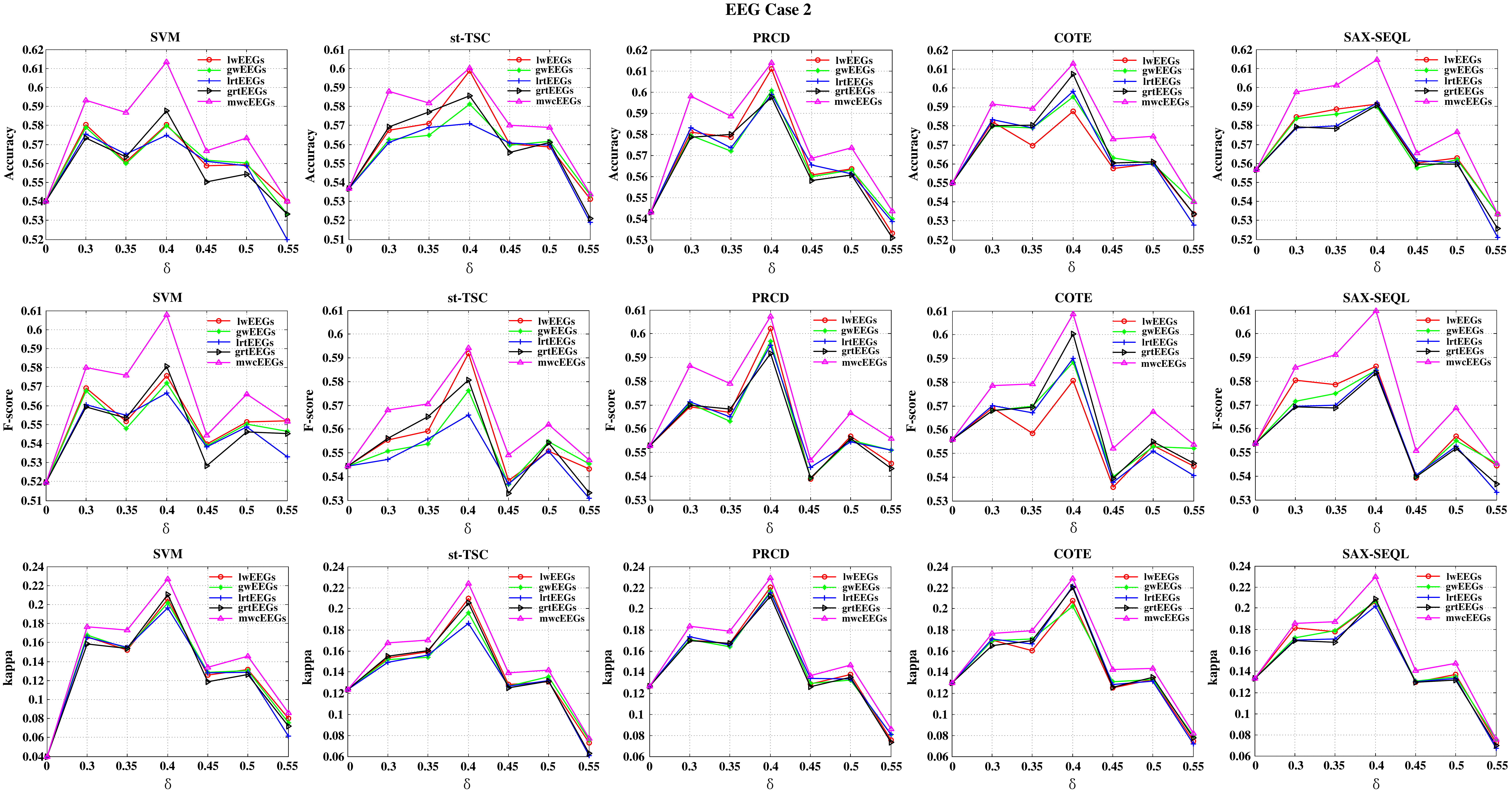}
\caption{Classification results of 5 classifiers with selected EEGs by different EEG selection methods on EEG Case 2.}
\label{fig:case2}\vspace{-0.2cm}
\end{figure*}

\begin{figure*}[!t]
\centering
\includegraphics[height=3.6in]{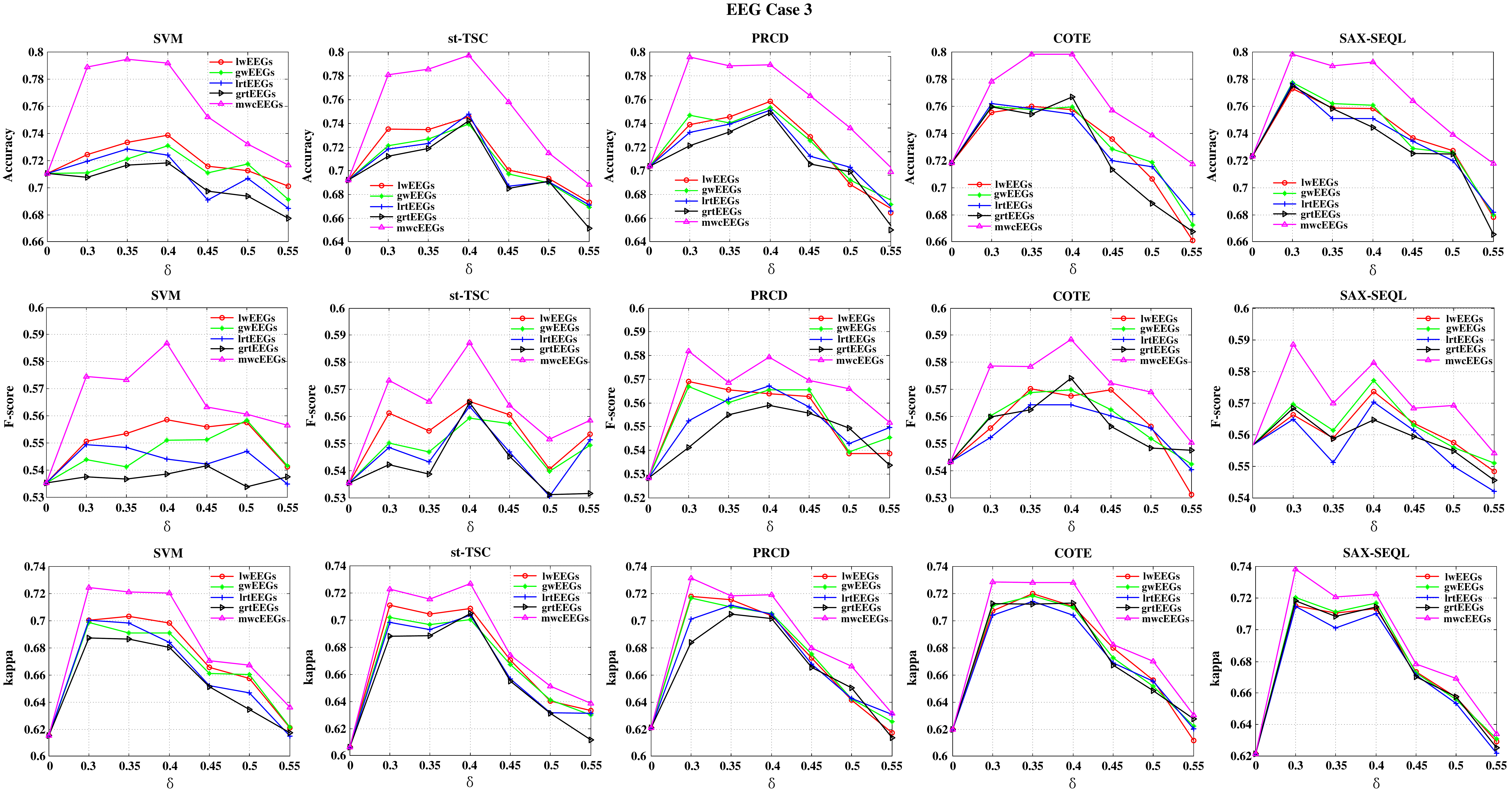}
\caption{Classification results of 5 classifiers with selected EEGs by different EEG selection methods on EEG Case 3.}
\label{fig:case3}\vspace{-0.2cm}
\end{figure*}

\subsection{Parameter setting}
For mwcEEGs, \begin{small}$\delta$\end{small} determines the selection of EEGs and affects the classification performance. The mwcEEGs with a smaller \begin{small}$\delta$\end{small} selects less discriminative but more general EEGs while mwcEEGs with a larger \begin{small}$\delta$\end{small} selects more discriminative but less general EEGs. Then the classification with classifiers are influenced by the selected EEGs, namely, the final classification results are affected by \begin{small}$\delta$\end{small}. And more distinguished EEGs with a large \begin{small}$\delta$\end{small} degrades the classification performance of classifiers because of low generality of selected discriminative EEGs which cannot represent the general EEG data. As a consequence, the mwcEEGs with a moderate or optimal \begin{small}$\delta$\end{small} selects a better amount of EEGs which balances the discrimination and generality, and seems to achieve better classification performance. Here we select the optimal \begin{small}$\delta$\end{small} based on the discrete probability distribution of similarities among EEG. Mathematically, the discrete probability distribution \begin{small}$d(\delta)$\end{small} is defined as
\begin{small}
\begin{equation}\label{distribution}
d(\delta)=\frac{n_{s\in[\delta_i,\delta_i+Interval]}}{N^2}
\end{equation}
\end{small}
where \begin{small}$N$\end{small} denotes the number of vertices (EEGs) in the similarity-weighted graph \begin{small}$G$\end{small}, \begin{small}$n$\end{small} denotes the number of similarities \begin{small}$s\in \bm S$\end{small} locate in \begin{small}$[\delta_i,\delta_i+Interval]$\end{small}.
\par With (\ref{distribution}), $\delta$ is chosen based on the (\ref{delta})
\begin{small}
\begin{equation}\label{delta}
\delta=f^{-1}(D)\in\Big\{\delta|f(\delta):\sum_{i\in{[0,\frac{1}{Interval}-1]}}^\frac{1}{Interval}d_i(\delta)=A\Big\}
\end{equation}
\end{small}
where \begin{small}$A\in (0,1)$\end{small} and a larger \begin{small}$A$\end{small} achieves a smaller \begin{small}$\delta$\end{small}.

Similarity threshold \begin{small}$\delta_{k=1,\cdots,m,m\geq 1}$\end{small} in the work for EEG selection with 3 datasets is respectively set based on the discrete probability distribution shown in (a) -- (e) of Figure \ref{fig:density}. (a) -- (e) respectively shows that most EEGs from the same dataset are similar with each other and the EEG similarity is displayed in (f).

To illustrate the influence of \begin{small}$\delta$\end{small} on mwcEEGs, we set \begin{small}$\delta\in\{0.3,0.35,0.4,0.45,0.5,0.55\}$\end{small} for 3 EEG Cases. The number of selected EEGs for lwEEGs, gwEEGs, lrtEEGs, and grtEEGs is set as same as mwcEEGs. For classifiers SVM, st-TSC, RPCD, COTE, and SAX-SEQL, we set the optimal parameters as same as the references set. As introduced before, data set in 3 cases is divided into 3 groups of training and testing data with the proportion of 2:1 based on the Hold-out strategy. All the methods are run on 3 groups of EEG data for each case and the results are averaged.

\subsection{Experimental Results and Discussion}
In this paper, we proposed the maximum weight clique inspired method mwcEEGs to select EEGs. To firmly establish the efficacy of our method, we compared the mwcEEGs with the state-of-the-art time series selection methods on several popular and newest time series classifiers for EEG classification. The experimental results with 3 cases (3 datasets) are shown in Figure \ref{fig:case1} -- \ref{fig:case3} respectively. We can see from the experimental results that with selected EEGs by mwcEEGs, the classification performance is improved compared that without selected EEGs. Moreover, a small or a moderate \begin{small}$\delta$\end{small} yields a better classification performance than a larger \begin{small}$\delta$\end{small}. In other word, a small or moderate \begin{small}$\delta$\end{small} calculated with (\ref{delta}) achieves a high-quality EEG classification. The reason is that the larger the \begin{small}$\delta$\end{small} is, more discriminative EEGs are selected by mwcEEGs. That is to say, the selected discriminative EEGs with a larger \begin{small}$\delta$\end{small} reduce more discriminative features of EEGs and cannot represent the general EEGs, so its classification results are probably lower than that with a smaller \begin{small}$\delta$\end{small} achieves or even lower than that without selected ones. As a consequence, mwcEEGs with small or moderate \begin{small}$\delta$\end{small} yields best classification results with respect to RI, F-score, and \emph{kappa}, which indicates that mwcEEGs is superior over the state-of-the-art time series selection methods for EEG classification on several promising classifiers.

\section{Conclusion} \label{conclusion}
This paper explores Brain EEG selection. Raw EEGs without selection contains many invalid/noisy data which degrades the corresponding learning performance. Since EEG is weak, complicated, fluctuated and with low signal-to-noise, conventional time series selection methods are not applicable for EEG selection. To address this issue, a novel approach (called mwcEEGs) based on maximum weight clique is proposed to select valid EEGs. The main idea of mwcEEGs is to map EEG selection to searching a family of cliques with maximum weights simultaneously combining edge weights and vertex weights in an improved Fr\'{e}chet distance-weighted EEG graph while reducing the influence of invalid/noisy EEGs according to similarity thresholds \begin{small}$\delta$\end{small}. The experimental comparisons with the state-of-the-art time series selection methods based on different evaluation criteria on real-world EEG data demonstrate the effectiveness of the mwcEEGs for EEG selection.

\section*{Acknowledgments}
This work was partially supported by National Natural Science Foundation of China (Nos. U1433116 and 61702355), Fundamental Research Funds for the Central Universities (Grant No. NP2017208), and Funding of Jiangsu Innovation Program for Graduate Eduction (Grant No. KYZZ16\_0171).

\bibliographystyle{IEEEtran}
\bibliography{Ref}
\end{document}